\documentclass[conference]{IEEEtran}
\usepackage{amsthm}
\usepackage{amsmath,graphicx}
\usepackage{amssymb,amsthm,mathtools}
\usepackage{bm}
\usepackage{empheq}
\usepackage{float}
\usepackage{booktabs}
\usepackage{multirow}
\usepackage{array}
\usepackage{algorithm}
\usepackage{algorithmic}
\usepackage{hyperref}
\usepackage{enumitem}
\usepackage{microtype}
\usepackage{cite}
\usepackage{pifont}
\hypersetup{hidelinks}

\newcommand{\ba}{\bm{a}}
\newcommand{\bx}{\bm{x}}
\newcommand{\by}{\bm{y}}
\newcommand{\bz}{\bm{z}}
\newcommand{\bw}{\bm{w}}

\newcommand{\bp}{\bm{p}}

\newcommand{\bG}{\bm{G}}

\newcommand{\bX}{\bm{X}}

\newcommand{\bI}{\bm{I}}
\newcommand{\bW}{\bm{W}}
\newcommand{\bM}{\bm{M}}

\newcommand{\bT}{\bm{T}}

\newtheorem{proposition}{Proposition}

\theoremstyle{definition}

\newcommand{\bSigma}{\bm{\Sigma}}

\DeclareMathOperator{\diag}{diag}

\DeclareMathOperator{\softmax}{softmax}

\newcommand{\Normal}{\mathcal{N}}

\newcommand{\R}{\mathbb{R}}

\newcommand{\Ind}{\mathbb{I}}

\newcommand{\cG}{\mathcal{G}}
\newcommand{\cV}{\mathcal{V}}
\newcommand{\cE}{\mathcal{E}}

\newcommand{\cD}{\mathcal{D}}
\newcommand{\cL}{\mathcal{L}}

\newcommand{\rev}[1]{#1}

\title{Online Bayesian Node Classification on Inductive Graphs under Distribution Shift}

\author{%
\IEEEauthorblockN{
Jinwen Xu\IEEEauthorrefmark{1},
Gonzalo Mateos Buckstein\IEEEauthorrefmark{2}, and
Qin Lu\IEEEauthorrefmark{1}
}

\IEEEauthorblockA{
\IEEEauthorrefmark{1}
School of Electrical and Computer Engineering,
University of Georgia, Athens, GA, USA
}

\IEEEauthorblockA{
\IEEEauthorrefmark{2}
Department of Electrical and Computer Engineering,
University of Rochester, Rochester, NY, USA
}

\IEEEauthorblockA{
\texttt{jinwen.xu@uga.edu,
gmateosb@ur.rochester.edu,
qin.lu@uga.edu}
}
}

\begin{document}
\maketitle
 
\begin{abstract}
On evolving graphs, node classifiers face two demands: inductive generalization to newly arriving nodes under distribution shift, and calibrated uncertainty for safety-sensitive applications. Standard graph neural networks (GNNs) are trained once and address neither. We adapt the Bayesian last layer (BLL) model, placing {\it random} last-layer (LL) parameters atop a {\it deterministic} GNN encoder for uncertainty quantification. The categorical softmax likelihood needed for classification breaks Gaussian conjugacy, so neither the training posterior nor the test-time streaming update admits a closed-form solution. We address both: a variational Bayesian last layer (VBLL) objective jointly trains the encoder and an approximate LL posterior by maximizing an evidence lower bound with \rev{Monte Carlo (MC)} expected log-likelihood; at test time, an online Laplace update on the LL posterior (encoder frozen) amounts to a power-prior Bayesian model with exponential forgetting and a Kullback-Leibler (KL) anchor to the training posterior. Across five node classification benchmarks under distribution shift, the online \rev{GVBLL} is the only model to win both accuracy and NLL on every dataset, with gains up to $+17$ points on \texttt{Cora} and $+14$ points on \texttt{ogbn-arxiv} over the strongest \rev{non-GVBLL} baseline, and remains competitive on calibration against MC Dropout, Deep Ensembles, Temperature Scaling, and Gaussian-process classifiers.
\end{abstract}
 
 
\section{Introduction}
\label{sec:intro}
 
Graph neural networks (GNNs)~\cite{kipf2017gcn,hamilton2017sage,velickovic2018gat,xu2019gin}\rev{\cite{ruiz2021gnn}} are the dominant approach for node classification, performing well when training and test nodes share a distribution. In real-world deployments, this assumption often fails: citation networks absorb papers from emerging research areas, social graphs encounter new user communities, and fraud detection systems face evolving attack patterns. In all these cases, the graph grows over time and the incoming nodes (inductively unseen at training time) exhibit distributional shift relative to the training data~\cite{wu2022invariance}.
This non-stationary streaming setting exposes two GNN classifier limitations. First, they are \emph{static}: once trained, the model parameters are fixed, with no mechanism to adapt as the data distribution drifts. Second, they are \emph{uncalibrated}: the softmax output of a deterministic classifier does not reflect true predictive uncertainty~\cite{wang2021confident}, which is critical for downstream decision-making on out-of-distribution nodes.\rev{ Conformal prediction can add coverage guarantees to a fixed predictor, but it does not by itself update the classifier or preserve a Bayesian posterior across a non-stationary stream~\cite{xu2026conformalized,xu2025online}.}\vspace{2pt}
 
\noindent\rev{\textbf{Related work.} Existing methods address these limitations only partially. Bayesian GNN approaches~\cite{zhang2019bgnn,hasanzadeh2020bayesian} place uncertainty over the full network but are expensive and require full retraining under distribution shift. Variational Bayesian last layers (VBLL)~\cite{harrison2024vbll} provide a lightweight alternative by restricting Bayesian treatment to the final classifier. Existing VBLL models rely on Gaussian--Gaussian conjugacy for closed-form regression updates; the categorical softmax likelihood breaks this conjugacy, leaving online VBLL classification open. General-purpose uncertainty-quantification (UQ) methods---Monte Carlo (MC) Dropout~\cite{gal2016dropout}, Deep Ensembles~\cite{lakshminarayanan2017ensemble}, and Temperature Scaling~\cite{guo2017calibration}---are architecture-agnostic but carry no posterior state across the stream. Deep kernel learning~\cite{wilson2016deep} replaces the classifier head with a Gaussian process atop the network features. BLL instead uses a linear kernel, trading expressiveness for closed-form regression updates and $O(d_eC)$ inference, where $d_e$ is the embedding dimension and $C$ is the number of classes. Continual-learning methods, including Elastic Weight Consolidation (EWC)~\cite{kirkpatrick2017ewc}, Variational Continual Learning (VCL)~\cite{nguyen2018vcl,wang2020streaming}, and Bayesian continual approaches~\cite{khan2021kap,kurle2020continual}, update the entire network rather than exploiting last-layer structure. Our KL anchor specializes EWC's Fisher-weighted quadratic penalty to the last layer, while our streaming update adopts VCL's convention of treating the previous approximate posterior as the new prior.}\vspace{2pt}
 
\noindent\rev{\textbf{Proposed approach and contributions.} We extend VBLL to online categorical node classification on graphs and call the resulting model graph VBLL (GVBLL). A variational Bayesian linear head atop a GNN encoder is trained jointly through an evidence lower bound (ELBO), using an MC expected log-likelihood and KL annealing. At test time, the encoder is frozen and only the lightweight LL posterior is updated as new nodes arrive. Because the categorical likelihood precludes the closed-form regression update, we derive a Laplace-based online rule with a forgetting factor $\lambda\in[0,1)$ and a KL anchor to the training posterior. The rule is equivalent to power-prior Bayesian updating~\cite{ibrahim2000power}.}\vspace{2pt}
 
\rev{Our contributions to trustworthy online learning over graphs are:}
\textbf{(C1)}~A \rev{GVBLL node-classification framework} with categorical softmax likelihood, MC-based variational training, and MAP/MC predictive inference (Section~\ref{sec:method}).
\textbf{(C2)}~A Laplace-based online posterior update with exponential forgetting (factor $\lambda \in [0,1)$, effective memory $1/(1{-}\lambda)$ batches) and a KL anchor, equivalent to power-prior Bayesian updating (Section~\ref{sec:online}).
\textbf{(C3)}~Experiments on $5$ node classification benchmarks under distribution shift against GNN, standard UQ, and Gaussian-process baselines, with gains in accuracy, calibration, and adaptation (Section~\ref{sec:experiments}).
\section{Variational Bayesian Last Layer for Classification}
\label{sec:method}
Standard GNN classifiers use a deterministic linear head, which provides no epistemic uncertainty and cannot be updated without retraining the entire network. We replace it with a Bayesian linear layer whose posterior is \rev{efficiently} updated at test time. Unlike VBLL in regression tasks, the categorical softmax likelihood breaks Gaussian conjugacy, making both the training objective and predictive distribution intractable. We address both via Monte Carlo estimation in training \rev{(Section~\ref{sec:training})} and inference \rev{(Section~\ref{sec:predictive})}, after defining the model \rev{(Section~\ref{sec:setup})}.
 
\subsection{Problem setup}
\label{sec:setup}
 
\rev{Consider a graph $\cG_t := (\cV_t, \cE_t)$ with node set $\cV_t$ and edge set $\cE_t$ at slot $t$. At $t=0$, a set $\cV_0$ of $N$ labeled training nodes is available, with feature matrix $\bX_0 \in \R^{N \times d}$ and labels $\by_0 \in [C]^N$, where $[C] \coloneqq \{1,\ldots,C\}$. At each subsequent step $t$, a batch $\Delta\cV_t$ of previously unseen nodes arrives with edges to existing nodes, so that $\cV_t=\cV_{t-1}\cup\Delta\cV_t$. The task is to classify every $v\in\Delta\cV_t$ with calibrated uncertainty. We focus on node arrivals. Handling node departures would require removing their likelihood contributions from the accumulated posterior and is outside the present scope.}
 
A GNN encoder $g_{\boldsymbol{\theta}}$ (e.g., GraphSAGE~\cite{hamilton2017sage}, GCN~\cite{kipf2017gcn}, GAT~\cite{velickovic2018gat}) with {\it trainable} parameters $\boldsymbol{\theta}$ maps each node $v$ to an embedding $\bz_v = g_{\boldsymbol{\theta}}(\bx_v, \cG) \in \R^{d_e}$, \rev{where $\bx_v\in\R^d$ is the feature vector of node $v$}. Atop this encoder, a Bayesian linear classification head defines the generative model
\begin{align}
    \bw_c &\sim \Normal(\bm{0}, \bI_{d_e}), \quad c \in [C], \label{eq:prior}\\
    y_v \mid \bz_v, \bW &\sim \text{Cat}\!\left(\softmax(\bz_v^\top \bW / \tau_v)\right), \label{eq:lik}
\end{align}
where $\bW := [\bw_1,\ldots,\bw_C] \in \R^{d_e \times C}$ is the weight matrix, $\text{Cat}(\bp)$ denotes the categorical distribution with class probabilities $\bp$, $\softmax(\ba)_c = \exp(a_c)/\sum_{c'}\exp(a_{c'})$, and $\tau_v = \text{softplus}(h_{\boldsymbol{\theta}_\tau}(\bz_v)) > 0$ is a per-node temperature \rev{implemented by a shared multilayer perceptron (MLP)} $h_{\boldsymbol{\theta}_\tau}$ with learnable parameters $\boldsymbol{\theta}_\tau$. 
\subsection{Variational training}
\label{sec:training}

Given the generative model in~\eqref{eq:prior}-\eqref{eq:lik}, the posterior of ${\bf W}$ cannot be computed in closed form because the categorical likelihood~\eqref{eq:lik} is not conjugate with the Gaussian prior~\eqref{eq:prior}. \rev{We therefore use variational training to jointly learn the GNN parameters $\boldsymbol{\theta}$, the temperature-head parameters $\boldsymbol{\theta}_\tau$, and an approximate posterior over ${\bf W}$.}
For tractability, the variational posterior is assumed to be a fully factorized Gaussian as
\begin{equation}
    q(\bW)  =\prod_{c=1}^{C}q(\bw_c) =\prod_{c=1}^{C} \Normal(\bw_c;\; \boldsymbol{\mu}_{c},\; \diag(\boldsymbol{\sigma}_c^2)),
\label{eq:posterior}
\end{equation}
where $\boldsymbol{\mu}_c$ and $\boldsymbol{\sigma}_c^2$ are the variational mean and diagonal variance for class $c$, both in $\R^{d_e}$. \rev{We stack the class means as $\bM := [\boldsymbol{\mu}_1, \ldots, \boldsymbol{\mu}_C] \in \R^{d_e \times C}$ and the class variances as $\bSigma := [\boldsymbol{\sigma}_1^2, \ldots, \boldsymbol{\sigma}_C^2] \in \R^{d_e \times C}_{>0}$. Thus, the row index represents an embedding dimension and the column index represents a class. The posterior contains $2d_eC$ variational parameters, and each class-specific weight vector $\bw_c$ has its own uncertainty profile.}
 
All parameters, including the encoder $\boldsymbol{\theta}$, posterior mean $\bM$, posterior log-variance $\log\bSigma$, and temperature head $\boldsymbol{\theta}_\tau$, are trained jointly by minimizing the negative evidence lower bound (ELBO) \vspace{-0.4cm}
\begin{align}
    \cL(\boldsymbol{\theta}, \bM, \bSigma,\boldsymbol{\theta}_\tau) &= \underbrace{-\frac{1}{N}\sum_{v\in {\cal V}_0} \mathbb{E}_{q(\bW)}\!\left[\log p(y_v \mid \bz_v, \bW)\right]}_{\text{Expected log-likelihood (ELL)}} \nonumber \\ &+ \frac{\alpha_{\text{KL}}}{N}\, D_{\text{KL}}(q \,\|\, p),
\label{eq:elbo}
\end{align}\vspace{-0.4cm}\\
\noindent where $\bz_v = g_{\boldsymbol{\theta}}(\bx_v, \cG_0)$ depends on $\boldsymbol{\theta}$ through the encoder, so the gradients of $\cL$ flow through both the Bayesian head parameters $(\bM, \bSigma)$ and the encoder weights $\boldsymbol{\theta}$. Here, the Kullback--Leibler (KL) divergence between the factorized posterior~\eqref{eq:posterior} and the standard normal prior~\eqref{eq:prior} \rev{is, after discarding additive and positive multiplicative constants,}\vspace{-0.2cm}
\begin{equation}
    D_{\text{KL}}(q \,\|\, p) \propto \sum_{j,c} \sigma_{j,c}^2 + \|\bM\|_F^2 - \sum_{j,c} \log \sigma_{j,c}^2,
\label{eq:kl}
\end{equation}\vspace{-0.4cm}\\
\noindent where $\|\bM\|_F^2 = \sum_{j,c}\mu_{j,c}^2$ is the squared Frobenius norm.
We anneal the KL weight $\alpha_{\text{KL}}$ from $0$ to $1$ over training: $\alpha_{\text{KL}} = 0$ for the first $E/3$ of $E$ total epochs, then linearly ramped to~1. When $\alpha_{\text{KL}} = 0$, the loss reduces to pure MC classification and $\bM$ learns freely; as $\alpha_{\text{KL}}$ increases, $\bSigma$ learns to balance classification accuracy (via the ELL) against regularization toward the prior (via the KL).

Unlike in regression VBLL~\cite{harrison2024vbll}, the categorical softmax likelihood renders the ELL analytically intractable. We estimate it via the reparameterization trick~\cite{kingma2014vae}: drawing $S$ weight samples $\bW^{(s)} = \bM + \bSigma^{1/2} \odot \bm{E}^{(s)}$, $\bm{E}^{(s)} \sim \Normal(\bm{0}, \bI_{d_e \times C})$, where $\odot$ denotes element-wise (Hadamard) multiplication, yields\vspace{-0.5cm}
\begin{equation}
    \text{ELL} =\frac{1}{NS}\sum_{v,s} \log \softmax(\bz_v^\top \bW^{(s)} / \tau_v)_{y_v}.
\label{eq:mc_ell}
\end{equation}\vspace{-0.5cm}\\
\textbf{Two-phase training protocol.}\; In Phase~1 (offline), we obtain the trained GNN encoder $g_{\boldsymbol{\theta}^*}$, the temperature head $h_{\boldsymbol{\theta}_\tau^*}$ and Bayesian head parameters $(\bM^*, \bSigma^*)$ by minimizing~\eqref{eq:elbo} using Adam for $E$ epochs. \rev{The MC ELL is essential in this phase: every sample $\bW^{(s)}$ depends on $\bSigma$ through the reparameterization, so both $\bM$ and $\bSigma$ receive informative gradients. A deterministic ELL evaluated only at $\bM$ would give $\bSigma$ zero likelihood gradient and leave it governed only by the prior.} At the end of Phase~1, the encoder and temperature head are frozen and the posterior $(\bM^*, \bSigma^*)$ is saved. In Phase~2 (online), only the posterior moments $(\bM^{(t)}, \bSigma^{(t)})$ are updated via the Laplace-based rule described in Section~\ref{sec:online}, while $g_{\boldsymbol{\theta}^*}$ and $h_{\boldsymbol{\theta}_\tau^*}$ remain fixed. \rev{For a batch $\Delta\cV_t$, the Bayesian-head computation scales as $O(|\Delta\cV_t|d_eC)$ and is independent of the total graph size once the embeddings have been computed.}
\subsection{Predictive inference}
\label{sec:predictive}
Given the trained posterior $q(\bW)$, the Bayesian predictive distribution for a test node $v$ is
$p(y_v{=}c \mid \bz_v) = \int \softmax(\bz_v^\top \bW / \tau_v)_c\; q(\bW)\, d\bW$,
which has no analytic expression because the softmax nonlinearity couples all classes inside the integral. We approximate it via MC sampling as\vspace{-0.2cm}
\begin{equation}
    \hat{p}(y_v{=}c \mid \bz_v) = \frac{1}{S}\sum_{s=1}^{S} \softmax(\bz_v^\top \bW^{(s)} / \tau_v)_c,
\label{eq:mc}
\end{equation}\vspace{-0.3cm}\\
\noindent where $\bW^{(s)} = \bM^* + \bSigma^{* 1/2} \odot \bm{E}^{(s)}$. When \rev{$\bSigma^* \to \bm{0}$}, all samples collapse to \rev{$\bM^*$} and~\eqref{eq:mc} recovers the maximum a posteriori (MAP) prediction $\softmax(\bz_v^\top \bM^* / \tau_v)$. When \rev{$\bSigma^*$} is large, the samples spread and the averaged prediction softens, expressing epistemic uncertainty. For point prediction (accuracy, negative log-likelihood), we use the MAP estimate directly to avoid the systematic softening inherent in posterior averaging.
 
On evolving graphs, the label distribution drifts as new nodes arrive from shifted communities, and the fixed posterior \rev{can produce miscalibrated predictions}. Retraining the \rev{pretrained encoder from scratch} is expensive and discards previously learned structure. \rev{Encoder fine-tuning or low-rank parameter updates provide a middle ground, but they add optimization and storage costs; we leave their integration to future work.} Because our BLL decouples representation from uncertainty, we update only the lightweight LL posterior, as detailed next.

\section{Online Posterior Update under Distribution Shift}
\label{sec:online}

\rev{After offline training, the model encounters a non-stationary stream of new nodes. The encoder $g_{\boldsymbol{\theta}^*}$ and temperature head $h_{\boldsymbol{\theta}_\tau^*}$ remain frozen. Only the last-layer posterior $(\bM, \bSigma) \in \R^{d_e \times C} \times \R^{d_e \times C}_{>0}$ is updated online. For a batch $\Delta\cV_t$, this head computation costs $O(|\Delta\cV_t|d_eC)$ and does not depend on the total number of graph nodes once embeddings are available. We also track the precision matrix $\bT \in \R^{d_e \times C}$, whose scalar entries satisfy $T_{j,c}:=1/\Sigma_{j,c}$. The exact categorical streaming update is intractable (Section~\ref{sec:exact}). We address it in three steps: (i) a Laplace approximation yields gradient and precision updates (Section~\ref{sec:laplace}); (ii) a power-prior model represents non-stationarity through discounting and anchoring (Section~\ref{sec:power}); and (iii) combining them gives the online update (Section~\ref{sec:update}).}
 
\subsection{Exact Bayesian update and its intractability}
\label{sec:exact}
 
Given all data observed up to step $t-1$ and a new batch \rev{$\cD_t = \{(\bx_v, y_v)\}_{v \in \Delta\cV_t}$}, where $\bz_v = g_{\boldsymbol{\theta}^*}(\bx_v, \cG_t)$ is computed from the frozen encoder, the exact Bayesian posterior update is\vspace{-0.4cm}
\begin{equation}
    p(\bW \mid \cD_{1:t}) = \frac{\prod_{v \in \cD_t} p(y_v \mid \bz_v, \bW) \;\cdot\; p(\bW \mid \cD_{1:t-1})}{p(\cD_t \mid \cD_{1:t-1})},
\label{eq:exact_update}
\end{equation}
\rev{Here, $p(\cD_t \mid \cD_{1:t-1}) = \int \prod_v p(y_v \mid \bz_v, \bW)\, p(\bW \mid \cD_{1:t-1})\, d\bW$ is the marginal likelihood. In regression VBLL~\cite{harrison2024vbll}, conjugacy yields a closed-form Gaussian posterior. The categorical softmax likelihood breaks this conjugacy. We index $q_t$ as the approximate posterior available immediately before batch $\cD_t$ is processed; hence, $q_t$ approximates $p(\bW\mid\cD_{1:t-1})$, and processing $\cD_t$ produces $q_{t+1}$. We therefore seek an efficient update that preserves the factorized Gaussian form and use a Laplace approximation.}
\subsection{Laplace approximation}
\label{sec:laplace}
 
We approximate~\eqref{eq:exact_update} by a Gaussian centered at the mode of the log-posterior with precision given by the Hessian, which we refer to as the Laplace approximation. For the categorical negative log-likelihood $\ell(\bW) = -\log p(y_v \mid \bz_v, \bW) = -\bz_v^\top \bw_{y_v} + \log\sum_{c'}\exp(\bz_v^\top \bw_{c'})$, the gradient and Hessian with respect to $\bw_c$ are\vspace{-0.3cm}
\begin{align}
    \nabla_{\bw_c} \ell = (\hat{p}_c - \Ind[c{=}y_v])\,\bz_v, \ 
    \nabla^2_{\bw_c} \ell = \hat{p}_c(1{-}\hat{p}_c)\,\bz_v\bz_v^\top, \label{eq:hessian}
\end{align}\vspace{-0.5cm}\\
\noindent where $\hat{p}_c = \softmax(\bz_v^\top \bW)_c$ and $\Ind[\cdot]$ is the indicator function, equal to one when its argument is true and zero otherwise. We take the diagonal approximation \rev{$\diag(\bz_v\bz_v^\top)=\bz_v\odot\bz_v=: \bz_v^2$} to preserve the factorized Gaussian form of $q$, yielding a per-class precision increment $\hat{p}_{v,c}(1{-}\hat{p}_{v,c})\, z_{v,j}^2$ at entry $(j,c)$. Applied repeatedly, this update accumulates precision from every batch, driving $\bT$ to grow monotonically and $\bSigma$ to shrink toward zero. In a stationary setting this is desirable. \rev{Under distribution shift, however, the posterior can concentrate around outdated beliefs and fail to track the changing label distribution.}
 
\subsection{Power-prior model for non-stationary streams}
\label{sec:power}
 
\rev{In a non-stationary stream, old observations may come from a distribution different from the current one, so their posterior contribution should be downweighted. The power-prior framework~\cite{ibrahim2000power} raises the previous posterior to a power $\lambda\in[0,1)$, thereby forgetting old data geometrically. We also add a KL-derived quadratic anchor of strength $\beta\geq0$ to limit drift from the training solution. Following streaming variational inference~\cite{nguyen2018vcl}, we use the previous approximate posterior $q_t(\bW)$ as the new prior:}
\begin{equation}
\begin{aligned}
    q_{t+1}(\bW) &\propto p(\cD_t \mid \bW)\left[q_t(\bW)\right]^\lambda \\
    &\quad\times \exp\!\left(-\tfrac{\beta}{2}\sum_{j,c}\tfrac{(W_{j,c} - M^*_{j,c})^2}{\Sigma^*_{j,c}}\right)\!,
\end{aligned}
\label{eq:power}
\end{equation}
where $q_t(\bW) = \prod_{c=1}^{C} \Normal(\bw_c;\; \boldsymbol{\mu}_c^{(t)},\; \diag(\boldsymbol{\sigma}_c^{2,(t)}))$ is the current approximate posterior and $(\bM^{*}, \bSigma^{*})$ is the posterior saved after offline training. \rev{The batch likelihood is $p(\cD_t\mid\bW)=\prod_{v\in\Delta\cV_t}p(y_v\mid\bz_v,\bW)$.}
 
Each component has a clear interpretation. The power $\lambda$ weights observations from $k$ batches ago by $\lambda^k$, yielding an effective memory horizon of $1/(1{-}\lambda)$ batches (at $\lambda = 0.995$, the model retains ${\sim}200$ batches of history). The anchor penalizes drift in proportion to training precision, so well-determined parameters (small $\Sigma^*_{j,c}$) are anchored strongly while uncertain ones adapt freely; $\beta = 0$ removes the anchor entirely, and large $\beta$ pins the weights near $\bM^*$. Setting $\lambda = 1$ and $\beta = 0$ recovers standard Bayesian updating~\eqref{eq:exact_update}.
\subsection{Online update equations}
\label{sec:update}
 
Applying the Laplace approximation from \rev{Section~\ref{sec:laplace}} to the power-prior model~\eqref{eq:power} yields closed-form update equations.
 
\begin{proposition}[Power-prior Laplace update]
\label{prop:power}
Under the diagonal Laplace approximation~\eqref{eq:hessian} of the power-prior posterior~\eqref{eq:power}, a damped Newton step from the current mode $\bM^{(t)}$ with step size $\eta$ gives 
\begin{align}
   & T^{(t+1)}_{j,c} = \lambda\, T^{(t)}_{j,c}\! +\! \eta \!\!\sum_{v \in \Delta\cV_t} \hat{p}_{v,c}(1{-}\hat{p}_{v,c})\, z_{v,j}^2 + \eta\beta/\Sigma^*_{j,c}, \label{eq:prec}\\
   & \bG = \sum_v \bz_v(\by_{\text{oh},v} - \hat{\bp}_v)^\top - \beta\,\big( (\bM^{(t)} - \bM^{*}) / \bSigma^{*} \big), \label{eq:grad_reg}\\
   & \bM^{(t+1)} = \bM^{(t)} + \text{clip}\!\left(\eta\, \bSigma^{(t+1)} \odot \bG,\; {-}\delta,\; \delta\right)\!, \label{eq:mean}
\end{align}
where \rev{$\bT = \mathbf{1}/\bSigma \in \R^{d_e \times C}$} is the per-class precision (element-wise reciprocal), $\hat{p}_{v,c} = \softmax(\bz_v^\top \bM^{(t)} / \tau_v)_c$, $\by_{\text{oh},v} \in \{0,1\}^C$ is the one-hot encoding of label $y_v$, and $\hat{\bp}_v \in \R^C$ is the full predicted probability vector. \rev{The operator $\text{clip}(\cdot,-\delta,\delta)$ clamps each entry to $[-\delta,\delta]$ for stability. Matrix division is element-wise between operands of the same shape.}
\end{proposition}
 
\begin{proof}
\rev{Taking the logarithm of~\eqref{eq:power}, differentiating with respect to its matrix argument $\bW$, and evaluating the gradient and negative Hessian at $\bM^{(t)}$ yield three terms.}
 
\noindent\emph{1. Gradients.} The data log-likelihood contributes $\sum_v \!\bz_v(\by_{\text{oh},v}\!\! - \hat{\bp}_v)^\top$ \rev{[cf.~\eqref{eq:lik}]}; the discounted prior $\lambda \log q_t(\bW)$ vanishes at $\bM^{(t)}$ since it is the mode of $q_t$; the KL anchor contributes $-\beta\,\big( (\bM^{(t)} - \bM^{*}) / \bSigma^{*} \big)$. Summing yields \rev{$\bG$} in~\eqref{eq:grad_reg}.
 
\noindent \emph{2. Precision (negative Hessian).} The data Hessian under the diagonal approximation~\eqref{eq:hessian} gives the per-class Fisher information $\hat{p}_{v,c}(1{-}\hat{p}_{v,c})\, z_{v,j}^2$ at entry $(j,c)$; the discounted prior contributes \rev{$\lambda T^{(t)}_{j,c}$}; the KL anchor contributes $\beta / \Sigma^{*}_{j,c}$. Weighting the new (data and anchor) contributions by step size $\eta$ yields~\eqref{eq:prec}.
 
\noindent \emph{3. A damped Newton step.} \rev{Multiplying the log-posterior gradient $\bG$ by the diagonal inverse precision $\bSigma^{(t+1)}$ gives the Newton direction.} A step of size $\eta$ therefore produces $\bM^{(t)} + \eta\,\bSigma^{(t+1)} \odot \rev{\bG}$; element-wise clipping for numerical stability gives~\eqref{eq:mean}.
\end{proof}
 
\rev{The factor $\hat{p}_{v,c}(1{-}\hat{p}_{v,c})$ in~\eqref{eq:prec} is the Fisher information of the categorical observation for class $c$, analogous to $1/\sigma_n^2$ in regression VBLL~\cite{harrison2024vbll}. Each class column $\bw_c$ accumulates precision at its own rate. A highly confident prediction ($\hat p_{v,c}\to1$ or $\hat p_{v,c}\to0$) contributes little Fisher information, whereas $\hat p_{v,c}=0.5$ contributes maximally. Finally, for a batch $\Delta\cV_t$, forming logits, the gradient, and the diagonal precision increment costs $O(|\Delta\cV_t|d_eC)$ time. The posterior state and its update require $O(d_eC)$ memory. These costs exclude the frozen encoder forward pass.}
 
\begin{algorithm}[t]
\caption{Online Bayesian Node Classification}
\label{alg:update}
\begin{algorithmic}[1]
\REQUIRE Trained encoder $g_{\boldsymbol{\theta}^*}$,  posterior $(\bM^*, \bSigma^*)$ with $\bSigma^* \in \R^{d_e \times C}$, temperature head $h_{\boldsymbol{\theta}^{*}_{\tau}}$
\REQUIRE Forgetting $\lambda$, learning rate $\eta$, anchor strength $\beta$, clip threshold $\delta$, stability $\epsilon$
\STATE $\bM^{(0)} \gets \bM^*$;\quad $T^{(0)}_{j,c} \gets 1/\Sigma^*_{j,c}$ \hfill\COMMENT{Initialize online state}
\FOR{each arriving batch $\cD_t = \{(\bx_v, y_v):v\in\Delta\cV_t\}$ on graph $\cG_t$}
    \STATE $\bz_v \gets g_{\boldsymbol{\theta}^*}(\bx_v, \cG_t)$ for all $v\in\Delta\cV_t$ \hfill\COMMENT{Encode (frozen)}
    \STATE $\hat{\bp}_v \gets \softmax(\bz_v^\top \bM^{(t)} / \tau_v)$ for all $v$ \hfill\COMMENT{Predict (MAP)}
    \STATE Evaluate $\text{Acc}_t, \text{NLL}_t, \text{ECE}_t$ \hfill\COMMENT{Before seeing labels}
    \STATE $\bG^{\text{data}} \gets \sum_{v} \bz_v\,(\by_{\text{oh},v} - \hat{\bp}_v)^\top$ \hfill\COMMENT{Log-likelihood gradient}
    \STATE $\bG \gets \bG^{\text{data}} - \beta(\bM^{(t)} - \bM^{*})/\bSigma^{*}$ \hfill\COMMENT{Add KL anchor}
    \STATE $T^{(t+1)}_{j,c} \gets \lambda\, T^{(t)}_{j,c} + \eta \sum_{v} \hat{p}_{v,c}(1{-}\hat{p}_{v,c})\, z_{v,j}^2 + \eta\beta/\Sigma^*_{j,c}$
    \STATE $\bSigma^{(t+1)} \gets 1 / (\bT^{(t+1)} + \epsilon)$ \hfill\COMMENT{For numerical stability}
    \STATE $\bM^{(t+1)} \gets \bM^{(t)} + \text{clip}\!\left(\eta\, \bSigma^{(t+1)} \odot \bG,\; {-}\delta,\; \delta\right)$
\ENDFOR
\end{algorithmic}
\end{algorithm}
\section{Experiments}
\label{sec:experiments}
 \begin{table*}[t]
\centering
\caption{Node classification results under random distribution shift. Best result per row in \textbf{bold}, second-best \underline{underlined}; ties are marked at both entries. Acc (\%$\uparrow$), NLL ($\downarrow$), ECE ($\downarrow$). GP-X denotes a Gaussian-process classifier with encoder X.}
\label{tab:results}
\vspace{2pt}
\setlength{\tabcolsep}{3pt}
\small
\begin{tabular}{@{}ll cccc ccc ccc cc@{}}
\toprule
& & \multicolumn{4}{c}{\textbf{GNN Baselines}} & \multicolumn{3}{c}{\textbf{Standard UQ}} & \multicolumn{3}{c}{\textbf{Gaussian Process}} & \multicolumn{2}{c}{\rev{\textbf{GVBLL (Ours)}}} \\
\cmidrule(lr){3-6} \cmidrule(lr){7-9} \cmidrule(lr){10-12} \cmidrule(lr){13-14}
\textbf{Dataset} & & GraphSAGE & GCN & GAT & GIN & MCDropout & DE & TempScale & GP-SAGE & GP-GCN & GP-GAT & Static & Online \\
\midrule
\multirow{3}{*}{Cora}
 & Acc & 51.46 & 62.19 & 63.74 & 54.81 & 63.32 & 64.73 & 63.74 & 30.17 & 43.94 & 30.17 & \underline{81.43} & \textbf{81.86} \\
 & NLL & 1.48 & 1.24 & 1.23 & 3.50 & 1.27 & 1.22 & 1.14 & 1.85 & 1.58 & 1.86 & \underline{0.67} & \textbf{0.59} \\
 & ECE & 0.16 & 0.18 & 0.22 & 0.28 & 0.25 & 0.23 & 0.15 & \textbf{0.07} & 0.13 & \textbf{0.07} & \underline{0.11} & \underline{0.11} \\
\midrule
\multirow{3}{*}{Cornell}
 & Acc & 57.50 & 41.39 & 46.67 & 45.56 & 60.83 & 56.39 & 60.28 & 45.56 & 43.33 & 43.89 & \underline{67.14} & \textbf{69.00} \\
 & NLL & 1.21 & 2.04 & 1.39 & 4.56 & 1.11 & 1.16 & 1.14 & 1.41 & 1.47 & 1.47 & \underline{0.97} & \textbf{0.95} \\
 & ECE & 0.36 & 0.42 & 0.33 & 0.44 & 0.32 & 0.32 & 0.33 & 0.29 & 0.34 & \textbf{0.23} & 0.29 & \underline{0.26} \\
\midrule
\multirow{3}{*}{Texas}
 & Acc & 65.56 & 58.89 & 61.11 & 55.56 & 70.56 & 70.56 & 69.17 & 68.89 & 56.39 & 55.28 & \underline{71.00} & \textbf{72.00} \\
 & NLL & 1.09 & 1.72 & 1.09 & 3.89 & 0.98 & 0.97 & 0.93 & 0.89 & 1.23 & 1.40 & \underline{0.84} & \textbf{0.80} \\
 & ECE & 0.31 & 0.39 & 0.31 & 0.38 & 0.34 & 0.35 & 0.32 & 0.28 & 0.35 & 0.32 & \underline{0.27} & \textbf{0.26} \\
\midrule
\multirow{3}{*}{Wisconsin}
 & Acc & 72.29 & 49.79 & 53.75 & 42.08 & 70.42 & 69.38 & 70.21 & \underline{76.04} & 56.46 & 45.83 & 75.57 & \textbf{78.00} \\
 & NLL & 0.93 & 1.77 & 1.24 & 4.89 & 0.95 & 0.91 & 0.90 & \textbf{0.77} & 1.25 & 1.38 & \underline{0.82} & \textbf{0.77}\\
 & ECE & 0.28 & 0.38 & 0.31 & 0.47 & 0.32 & 0.29 & 0.29 & 0.26 & 0.34 & \underline{0.22} & \underline{0.22} & \textbf{0.21} \\
\midrule
\multirow{3}{*}{ogbn-arxiv}
 & Acc & 38.14 & 39.91 & 29.67 & 22.07 & 38.23 & 40.23 & 38.93 & 32.20 & 32.45 & 25.30 & \underline{50.49} & \textbf{54.31} \\
 & NLL & 2.27 & 2.23 & 2.43 & 2.83 & 2.31 & 2.25 & 2.22 & 2.52 & 2.53 & 2.75 & \underline{2.16} & \textbf{1.64} \\
 & ECE & 0.12 & 0.13 & \textbf{0.07} & \underline{0.09} & 0.14 & 0.15 & 0.10 & 0.13 & 0.12 & \underline{0.09} & 0.23 & \underline{0.09} \\
\bottomrule
\end{tabular}
\end{table*}
\subsection{Setup}
 
\rev{We evaluate on five node-classification benchmarks. \emph{Cora} ($2{,}708$ nodes, $7$ classes) uses a $5\%$ stratified split with $135$ training nodes and $2{,}573$ streaming-test nodes. The WebKB datasets \emph{Cornell} and \emph{Texas} (both $183$ nodes) and \emph{Wisconsin} ($251$ nodes) each have $5$ classes and use $20\%$ stratified training splits, leaving $147$--$201$ streaming nodes. \emph{ogbn-arxiv} ($169{,}343$ nodes, $40$ classes) uses the official temporal split: papers before 2018 are used for training, 2018 papers for validation, and 2019--2020 papers for the stream. We use 30 streaming steps for Cora (about 86 nodes per step), 20 for each WebKB graph (about 7--11 nodes per step), and 100 for ogbn-arxiv (about 486 nodes per step). Batch membership is fixed before evaluation. The encoder sees neither streaming features nor labels during training.}

\rev{We compare four model families: (i) four standalone GNN baselines (GraphSAGE, GCN, GAT, and GIN), trained once without test-time adaptation; (ii) three UQ baselines---MC Dropout~\cite{gal2016dropout}, Deep Ensembles (DE)~\cite{lakshminarayanan2017ensemble}, and Temperature Scaling~\cite{guo2017calibration}; (iii) Gaussian-process classifiers~\cite{wilson2016deep} with GraphSAGE, GCN, or GAT backbones; and (iv) our \textbf{GVBLL-Static} and \textbf{GVBLL-Online}. The proposed GVBLL models and the UQ baselines use the same GraphSAGE encoder and training data; GVBLL-Online updates its posterior after every batch using Algorithm~\ref{alg:update}. Results are averaged over $10$ random seeds.}
 
We report three metrics averaged over all streaming batches: classification accuracy (Acc, \%$\uparrow$), negative log-likelihood (NLL, $\downarrow$), and expected calibration error (ECE, $\downarrow$). \rev{NLL and ECE evaluate probabilistic calibration. For batch $\Delta\cV_t$, $\text{NLL}:=-|\Delta\cV_t|^{-1}\sum_{v\in\Delta\cV_t}\log\hat p(y_v\mid\bz_v)$. ECE is $\text{ECE}:=\sum_{b=1}^{B_{\mathrm{cal}}}|I_b|/|\Delta\cV_t|\,|\text{acc}(I_b)-\text{conf}(I_b)|$, where the $I_b$ are equal-mass confidence bins.}
 
\subsection{Results}
 
Table~\ref{tab:results} reports node classification under random batching. The comparisons below isolate two contributions: the Bayesian head \rev{(GVBLL-Static versus GNN baselines)} and the online update \rev{(GVBLL-Online versus GVBLL-Static)}.

\vspace{2pt}\noindent\rev{\textbf{The head explains most gains when labels are scarce.} On Cora, every non-GVBLL baseline reaches at most $64.73\%$ accuracy and an NLL no better than $1.14$, whereas GVBLL-Static reaches $81.43\%$ accuracy and $0.67$ NLL before any online update. The same pattern holds on the WebKB graphs, which contain only ${\sim}36$--$50$ training nodes. Learning $\bSigma$ jointly with $\bM$ avoids reducing the predictive distribution to an overconfident deterministic-softmax point estimate.}
 
\vspace{2pt}\noindent\rev{\textbf{Online adaptation helps most under pronounced drift.} GVBLL-Online improves on GVBLL-Static on every dataset. The accuracy gain is modest on Cora ($+0.43$ points) and Texas ($+1.00$), but larger on Wisconsin ($+2.43$) and ogbn-arxiv ($+3.82$). On ogbn-arxiv, GVBLL-Online exceeds the strongest non-GVBLL baseline by $14.1$ accuracy points and reduces NLL from $2.22$ to $1.64$. The update in Proposition~\ref{prop:power} explains this behavior: $\lambda$ discounts stale precision, the Laplace term adds information from the new batch, and the KL anchor limits drift from $\bM^*$.}
 
\vspace{2pt}\noindent\rev{\textbf{Calibration metrics must be interpreted with accuracy.} GP-SAGE ties GVBLL-Online on Wisconsin NLL ($0.77$), and GP-GAT obtains the lowest Cornell ECE ($0.23$ versus $0.26$). On Cora, however, GP-SAGE and GP-GAT obtain only ${\sim}30\%$ accuracy despite an ECE of $0.07$. GAT similarly obtains $29.67\%$ accuracy and $0.07$ ECE on ogbn-arxiv. Near-uniform predictions can therefore appear calibrated while providing poor classification. Across all datasets, GVBLL-Online achieves the best Acc and NLL, while its ECE remains within $0.05$ of the best baseline.}

\section{Conclusion}
\label{sec:conclusion}
 
We presented a \rev{GVBLL} framework for online node classification on evolving graphs under distribution shift. A GNN encoder is augmented with a Bayesian classification head trained via ELBO with \rev{MC} expected log-likelihood and KL annealing, jointly learning the classification weights and their per-class posterior variance. At test time, the lightweight last-layer posterior is updated by a Laplace approximation of a power-prior Bayesian model with exponential forgetting and a KL anchor to the training posterior. The approach is encoder-agnostic and scales to graphs with over $100$K nodes.

\rev{The present study assumes node arrivals, fixed batch schedules, and a frozen encoder; it does not provide formal predictive-coverage guarantees. Future work will study node departures, sensitivity to online batch size and offline training-set size, adaptive conformal calibration, and encoder fine-tuning or low-rank updates. The same Bayesian-head construction can also be extended to streaming graph classification by replacing the node-level equivariant encoder with a permutation-invariant graph-level encoder. Finally, an explicit drift model, such as a slowly varying random walk, may enable tracking-error or regret analysis for the online posterior update.}

\bibliographystyle{IEEEtran}
\bibliography{strings,refs}

\end{document}